\documentclass[12pt, a4paper]{article}
\usepackage[a4paper,
top=2.5cm, bottom=2.5cm,
left=2.8cm, right=2.8cm,
headheight=14pt]{geometry}
\usepackage[utf8]{inputenc}
\usepackage[T1]{fontenc}
\usepackage{amsmath,amssymb,amsthm,amsfonts,mathtools}
\usepackage{booktabs,array}
\usepackage{xcolor}
\usepackage{tikz}
\usetikzlibrary{positioning,arrows.meta,calc,shapes.geometric,fit,backgrounds,decorations.pathreplacing,matrix}
\usepackage{enumitem}
\usepackage{fancyhdr}
\usepackage{hyperref}
\hypersetup{colorlinks=true,linkcolor=blue!60!black,citecolor=blue!60!black,urlcolor=blue!60!black}
\usepackage{cleveref}
\theoremstyle{plain}
\newtheorem{theorem}{Theorem}[section]
\newtheorem{proposition}[theorem]{Proposition}
\newtheorem{lemma}[theorem]{Lemma}
\newtheorem{corollary}[theorem]{Corollary}
\theoremstyle{definition}
\newtheorem{definition}[theorem]{Definition}
\newtheorem{example}[theorem]{Example}
\theoremstyle{remark}
\newtheorem{remark}[theorem]{Remark}
\newtheorem{assumption}[theorem]{Assumption}
\crefname{assumption}{Assumption}{Assumptions}
\Crefname{assumption}{Assumption}{Assumptions}
\DeclareMathOperator{\diag}{diag}

\newcommand{\R}{\mathbb{R}}
\newcommand{\NN}{\mathcal{A}}
\newcommand{\FF}{\mathcal{F}}
\newcommand{\FFA}{\mathcal{F}^{*}}
\newcommand{\had}{\odot}
\newcommand{\transp}{{}^{\top}}
\newcommand{\loss}{J}
\newcommand{\pp}[2]{\dfrac{\partial #1}{\partial #2}}
\newcommand{\cB}{\mathcal{B}}
\newcommand{\X}{\mathbf{X}}
\newcommand{\Y}{\mathbf{Y}}
\newcommand{\Xs}{\mathbf{X}_{*}}
\newcommand{\Ys}{\mathbf{Y}_{*}}
\newcommand{\bG}{\mathbf{G}}
\newcommand{\norm}[1]{\left\lVert#1\right\rVert}
\newcommand{\Sigmap}{\mathbf{\Sigma}'}
\newcommand{\Dell}{\mathbf{D}}
\title{Backpropagation as a Nilpotent Linear System}
\author{%
	\sc  Ahmed Boughammoura\thanks{Correspondence to   ahmed.boughammoura@gmail.com}
	\\
	{\small \sc Higher Institute of Informatics and Mathematics of Monastir, }\\ {\small\sc University of Monastir, 5000 Monastir, Tunisia.}  
}

\date{\today}
\begin{document}
\maketitle
\thispagestyle{fancy}
\begin{abstract}
\noindent
Backpropagation is the computational engine of deep learning, yet its
mathematical structure is typically treated as a procedural traversal of
computational graphs. We present a global operator theory of the
\emph{F-adjoint} framework, which reformulates the layerwise backward
recursion of an $L$-depth feedforward network into a single linear system
$(I-\cB)\Xs=\bG$, where $\bG$ is a source vector. We prove that the global backward operator $\cB$ is
strictly block upper-triangular and nilpotent of index at most $L$. This
nilpotency guarantees the exact termination of the Neumann series solution
after at most $L$ terms, revealing classical backpropagation to be
mathematically equivalent to block back-substitution on an upper bidiagonal
system. We formalise \emph{F-symmetry}---the condition in which the backward pass
perfectly mirrors the forward pass---identifying orthogonal weight
matrices as canonical examples. Through worked numerical examples,
we demonstrate how this operator perspective exposes the single-path
collapse of strictly feedforward networks and its breakdown in residual
architectures. Finally, we leverage this compositional structure to
rigorously derive the mechanics of residual networks (gradient highways)
and transfer learning (gradient truncation). This framework elevates
backpropagation from an algorithmic recipe to a global nilpotent-operator
formulation.
\end{abstract}
\bigskip
% ---------------------------------------------------------------
\section{Introduction}\label{sec:intro}
% ---------------------------------------------------------------
An \emph{artificial neural network} (ANN) is a parametric function inspired
by biological neural systems~\cite{mcculloch1943}. In a fully connected,
layered feedforward architecture, the network maps an input vector
$x\in\R^{N_0}$ to an output vector $f(x)\in\R^{N_L}$ through a cascade of
affine and nonlinear transformations indexed by depth $\ell=1,\ldots,L$,
and is widely used for classification, pattern recognition, and
multivariate data analysis~\cite{basheer2000artificial}. The standard
training procedure is \emph{backpropagation}~\cite{rumelhart1986}, which
computes the gradient of a loss function $\loss(f(x),y)$ with respect to
every weight matrix $W^{(\ell)}$ by a recursive application of the chain
rule in reverse order through the layers. Backpropagation was popularised
by~\cite{rumelhart1986} and independently rediscovered in the specific
context of convolutional recognition systems~\cite{lecun1989backpropagation};
it remains the computational backbone of modern deep
learning~\cite{goodfellow2016deep}.
Despite its ubiquity, the mathematical structure underlying backpropagation
has traditionally received a fairly informal, procedural treatment:
gradients flow backward from the output layer, with each layer's adjoint
obtained from its successor by the chain rule. This view is computationally
efficient and pedagogically transparent, but it obscures the global
algebraic structure of the entire backward pass. Automatic differentiation
(AD) theory situates backpropagation within the broader class of
reverse-mode adjoint computations on computational
graphs~\cite{griewank2008evaluating,baydin2018automatic}, and adjoint
methods have a long history in optimal control and PDE-constrained
optimisation~\cite{pontryagin1987mathematical,hinze2008optimization}.
While reverse-mode AD provides a general framework for computing gradients
on arbitrary computational graphs, our contribution is the explicit
identification of the nilpotent operator structure specific to layered
networks. More recently, global formulations of network dynamics have been
explored for neural ordinary differential equations~\cite{chen2018neural}
and deep equilibrium models~\cite{bai2019deep}, both of which treat the
whole network as a single mathematical object. However, none of these
approaches fully exploits the specific block structure induced by a
\emph{finite-depth, layered} feedforward architecture.
Boughammoura~\cite{boughammoura2023two,boughammoura2023,boughammoura2024learning} filled this gap
by introducing the notion of \emph{F-propagation} and \emph{F-adjoint},
which rephrase the backward pass in a two-step recursive form perfectly
symmetric to the two-step form of the forward pass. The central theorem of
that work states that \emph{backpropagation is the F-adjoint of the
F-propagation}. The present paper builds on this layerwise result and
develops it into a \emph{global operator} theory: stacking every layer's
adjoint variables into a single vector $\Xs$, the layerwise F-adjoint recursion
becomes a single linear fixed-point equation $(I-\cB)\Xs=\bG$ governed by a
nilpotent operator $\cB$ and an injected source vector $\bG$. This reformulation makes the triangular,
finite-termination character of backpropagation explicit and provides a
uniform vantage point from which to derive worked examples, structural
symmetry results, and architecture-level applications.
\paragraph{Contributions and scope.}
This paper is a focused, self-contained exposition of the F-propagation
and F-adjoint framework, restricted to four themes:
\begin{enumerate}[leftmargin=*,label=(\roman*)]
\item a rigorous statement of the F-propagation/F-adjoint definitions
and the theorem identifying backpropagation with the F-adjoint of the
F-propagation (\Cref{sec:methodology});
\item a global operator formulation of the same statement---the
linear system $(I-\cB)\Xs=\bG$, the nilpotency of $\cB$, the Neumann
series solution, and its equivalence with block back-substitution and
with the classical layerwise recursion, together with a global
treatment of F-symmetric networks (\Cref{sec:methodology});
\item fully worked, numerically verified examples that make every
object above concrete (\Cref{sec:results});
\item two applications of the compositional structure of the
F-adjoint: residual networks and transfer learning (\Cref{sec:applications}).
\end{enumerate}
Material outside this scope---in particular the equilibrium F-adjoint
dynamics, local biologically plausible learning rules, and their
experimental evaluation~\cite{boughammoura2024learning}---lies beyond the
present paper and is not treated here.
\paragraph{Organisation.}
\Cref{sec:methodology} develops the methodology: notation and network
architecture (\Cref{subsec:notation}), the two-step rule for
backpropagation (\Cref{subsec:twostep}), the F-propagation and F-adjoint
definitions together with the main equivalence theorem
(\Cref{subsec:fadjoint-def}), the global operator construction and its
nilpotency (\Cref{subsec:global}), the block back-substitution and path
decomposition viewpoint (\Cref{subsec:backsub}), and the global theory of
F-symmetry (\Cref{subsec:fsym}). \Cref{sec:results} presents worked and
comparative examples. \Cref{sec:applications} discusses two applications
of the compositional structure. \Cref{sec:discussion} draws out structural
consequences of the equivalence, explicitly identifies assumptions and
limitations, and connects the framework to neighbouring theories.
\Cref{sec:conclusion} concludes.
% ---------------------------------------------------------------
\section{Methodology}\label{sec:methodology}
% ---------------------------------------------------------------
\subsection{Notation and network architecture}\label{subsec:notation}
Every vector $X\in\R^n$ is understood as a \emph{column} vector,
$X=(X_1,\ldots,X_n)\transp$. For a coordinate-wise activation
$\sigma:\R^n\to\R^n$ built from scalar maps $\sigma_i:\R\to\R$,
\begin{equation}\label{eq:sigma-def}
\sigma(X) := \bigl(\sigma_1(X_1),\ldots,\sigma_n(X_n)\bigr)\transp,
\end{equation}
so $[\sigma(X)]_i=\sigma_i(X_i)$, and the derivative $\sigma'(Y)\in\R^n$ is
defined component-wise, $[\sigma'(Y)]_i=\sigma_i'(Y_i)$. We write $A\had B$
for the Hadamard (element-wise) product of two vectors or matrices of the
same shape. \Cref{tab:notation} summarises the network-level symbols used
throughout.
\begin{table}[htp]
\centering
\def\arraystretch{1.3}
\begin{tabular}{|l|l|}
\hline
\textbf{Symbol} & \textbf{Description} \\
\hline
$W^{(\ell)}$
& Weight matrix of layer $\ell$, $W^{(\ell)}\in\R^{N_{\ell}\times N_{\ell-1}}$ \\
\hline
$Y^{(\ell)}$
& Pre-activation vector of layer $\ell$, $Y^{(\ell)} = W^{(\ell)}X^{(\ell-1)}\in\R^{N_{\ell}}$ \\
\hline
$X^{(\ell)}$
& Post-activation vector of layer $\ell$, $X^{(\ell)} = \sigma(Y^{(\ell)})\in\R^{N_{\ell}}$ \\
\hline
$X_*^{(\ell)}$
& Adjoint activation of layer $\ell$, $X_*^{(\ell)}=\partial\loss/\partial X^{(\ell)}\in\R^{N_{\ell}}$ \\
\hline
$Y_*^{(\ell)}$
& Adjoint pre-activation of layer $\ell$, $Y_*^{(\ell)}=\partial\loss/\partial Y^{(\ell)}\in\R^{N_{\ell}}$ \\
\hline
$\sigma$
& Pointwise activation function, $\sigma:\R\to\R$ \\
\hline
\end{tabular}
\caption{Notation for a fully connected feedforward network.}
\label{tab:notation}
\end{table}
We consider fully connected feedforward neural networks with $L$ layers,
denoted
\begin{equation}\label{eq:arch}
\NN[N_0, N_1, \ldots, N_L],
\end{equation}
where $N_0$ is the input dimension, $N_\ell$ is the width of hidden layer
$\ell$ ($\ell=1,\ldots,L-1$), and $N_L$ is the output dimension; $L$ is the
\emph{depth} of the network. For each layer $\ell=1,\ldots,L$, the weight
matrix is $W^{(\ell)}\in\R^{N_\ell\times N_{\ell-1}}$, and the forward pass
is
\begin{equation}\label{eq:forward}
X^{(0)}=x,\qquad
Y^{(\ell)} = W^{(\ell)} X^{(\ell-1)},\qquad
X^{(\ell)} = \sigma(Y^{(\ell)}),\qquad \ell=1,\ldots,L,
\end{equation}
and $f(x):=X^{(L)}$ is the network output. Training minimises a differentiable
loss $\loss(f(x),y)$, where $y\in\R^{N_L}$ is a target.
\begin{remark}[Bias absorption]\label{rem:bias}
Biases can be incorporated into the weight matrices by the
augmented-input convention: set $X^{(\ell)} =
(X^{(\ell)}_1,\ldots,X^{(\ell)}_{N_\ell-1},1)\transp$ for
$0\le\ell\le L-1$, with the last activation held fixed and equal to
$1$. For notational simplicity we work throughout with the already
bias-absorbed form $W^{(\ell)}\in\R^{N_\ell\times N_{\ell-1}}$
of~\eqref{eq:forward}.
\end{remark}
\subsection{Derivative conventions}\label{subsec:deriv}
Following~\cite{boughammoura2023}, we use the \emph{denominator-layout}
convention throughout (see~\cite{ye2022} for background). For a
differentiable map $F:\R\to\R^{m\times n}$,
\begin{equation}\label{eq:dFdx}
\pp{F}{x}
= \left(\pp{F_{ij}(x)}{x}\right)_{\substack{1\le i\le m\\1\le j\le n}}.
\end{equation}
For a differentiable scalar $F:\R^{m\times n}\to\R$,
\begin{equation}\label{eq:dFdX}
\pp{F}{X}
= \left(\pp{F(X)}{X_{ij}}\right)_{\substack{1\le i\le m\\1\le j\le n}}
\in\R^{m\times n}.
\end{equation}
Let $W\in\R^{q\times m}$ and $X\in\R^m$. These conventions give the
essential identities
\begin{equation}\label{eq:id1}
\pp{(WX)}{X} = W\transp, \qquad
\pp{F}{W} = \pp{F}{Z}\,X\transp,
\quad Z:=WX,\;F\text{ scalar-valued.}
\end{equation}
The chain rule then gives, for $F:(W,X)\mapsto Z:=WX\mapsto F(Z)$,
\begin{equation}\label{eq:chain-wx}
\pp{F}{X} = W\transp\pp{F}{Z}
\qquad\text{and}\qquad
\pp{F}{W} = \pp{F}{Z}\,X\transp;
\end{equation}
and for $F:\R^n\ni Y\mapsto X:=\sigma(Y)\mapsto F(X)$,
\begin{equation}\label{eq:chain-sigma}
\pp{F}{Y} = \pp{F}{X}\had\sigma'(Y).
\end{equation}
\subsection{The two-step rule for backpropagation}\label{subsec:twostep}
Following~\cite{boughammoura2023}, we introduce the shorthand
\begin{equation}\label{eq:star-notation}
X^{(\ell)}_* := \pp{\loss(f(x),y)}{X^{(\ell)}},\quad
Y^{(\ell)}_* := \pp{\loss(f(x),y)}{Y^{(\ell)}},\quad
W^{(\ell)}_* := \pp{\loss(f(x),y)}{W^{(\ell)}}.
\end{equation}
\begin{theorem}[Two-step rule for backpropagation
{\cite{boughammoura2023two,boughammoura2023}}]
\label{thm:twostep}
Let $\NN[N_0,\ldots,N_L]$ be a deep neural network with weight
matrices $W^{(\ell)}\in\R^{N_\ell\times N_{\ell-1}}$ and pointwise
differentiable activation $\sigma$. Define the forward pass
by~\eqref{eq:forward}, let $\loss(f(x),y)$ be a differentiable loss
with $f(x)=X^{(L)}$, and set
\begin{equation}\label{eq:init-backward}
X^{(L)}_* = \pp{\loss(f(x),y)}{X^{(L)}}.
\end{equation}
Then for $\ell=L,L-1,\ldots,1$:
\begin{align}
Y^{(\ell)}_* &= X^{(\ell)}_*\had\sigma'(Y^{(\ell)}), \label{eq:back1}\\
X^{(\ell-1)}_* &= (W^{(\ell)})\transp Y^{(\ell)}_*, \label{eq:back2}
\end{align}
and the weight gradient is
\begin{equation}\label{eq:weight-grad}
W^{(\ell)}_* = Y^{(\ell)}_*(X^{(\ell-1)})\transp.
\end{equation}
\end{theorem}
\begin{proof}
Fix $\ell\in\{1,\ldots,L\}$.
\noindent\textit{Step 1 (pre-activation gradient).}
Since $X^{(\ell)}=\sigma(Y^{(\ell)})$, identity~\eqref{eq:chain-sigma} gives
$Y^{(\ell)}_* = X^{(\ell)}_*\had\sigma'(Y^{(\ell)})$, establishing~\eqref{eq:back1}.
\noindent\textit{Step 2 (activation gradient).}
Since $Y^{(\ell)} = W^{(\ell)} X^{(\ell-1)}$, identity~\eqref{eq:chain-wx}
gives $X^{(\ell-1)}_* = (W^{(\ell)})\transp Y^{(\ell)}_*$, which is~\eqref{eq:back2}.
\noindent\textit{Step 3 (weight gradient).}
By~\eqref{eq:chain-wx}, $W^{(\ell)}_* = Y^{(\ell)}_*(X^{(\ell-1)})\transp$,
giving~\eqref{eq:weight-grad}.
\end{proof}
\begin{remark}[Forward--backward symmetry]\label{rem:two-step-symmetry}
Equations~\eqref{eq:forward} and~\eqref{eq:back1}--\eqref{eq:back2}
are structurally parallel. The forward pass produces the ordered
family $(X^{(0)},Y^{(1)}, \ldots,Y^{(L)},X^{(L)})$
via the alternating maps $W^{(\ell)}(\cdot)$ and $\sigma(\cdot)$,
while the backward pass produces the ordered family
$(X^{(L)}_*,Y^{(L)}_*,\ldots,Y^{(1)}_*,X^{(0)}_*)$
via the adjoint alternating maps $(\cdot)\had\sigma'(Y^{(\ell)})$
and $(W^{(\ell)})\transp(\cdot)$. This structural duality is the
motivating observation behind the F-adjoint framework.
\end{remark}
\subsection{F-propagation and F-adjoint}\label{subsec:fadjoint-def}
\begin{definition}[F-propagation {\cite{boughammoura2023,boughammoura2024learning}}]
\label{def:fprop}
Let $X^{(0)}\in\R^{N_0}$ be the input, $\sigma$ a coordinate-wise
map, and $W^{(\ell)}\in\R^{N_\ell\times N_{\ell-1}}$ for all
$1\le\ell\le L$. The \emph{F-propagation} through
$\NN[N_0,\ldots,N_L]$ is the ordered family of vectors
\begin{equation}\label{eq:fprop}
\FF :=
\bigl(X^{(0)},\;Y^{(1)},\;X^{(1)},\;\ldots,\;X^{(L-1)},\;Y^{(L)},\;X^{(L)}\bigr)
\end{equation}
satisfying, for $\ell=1,\ldots,L$,
\begin{equation}\label{eq:fprop-rec}
Y^{(\ell)} = W^{(\ell)} X^{(\ell-1)},\qquad
X^{(\ell)} = \sigma(Y^{(\ell)}).
\end{equation}
\end{definition}
\begin{definition}[F-adjoint {\cite{boughammoura2023,boughammoura2024learning}}]
\label{def:fadjoint}
Let $X^{(L)}_*\in\R^{N_L}$ be a given seed and let $\FF$ be an
F-propagation through $\NN[N_0,\ldots,N_L]$. The \emph{F-adjoint}
$\FFA$ of $\FF$ with seed $X^{(L)}_*$ is the ordered family
\begin{equation}\label{eq:fadjoint}
\FFA :=
\bigl(X^{(L)}_*,\;Y^{(L)}_*,\;X^{(L-1)}_*,\;\ldots,\;
X^{(1)}_*,\;Y^{(1)}_*,\;X^{(0)}_*\bigr),
\end{equation}
satisfying, for $\ell=L,L-1,\ldots,1$,
\begin{equation}\label{eq:fadjoint-rec}
Y^{(\ell)}_* = X^{(\ell)}_*\had\sigma'(Y^{(\ell)}),\qquad
X^{(\ell-1)}_* = (W^{(\ell)})\transp Y^{(\ell)}_*.
\end{equation}
\end{definition}
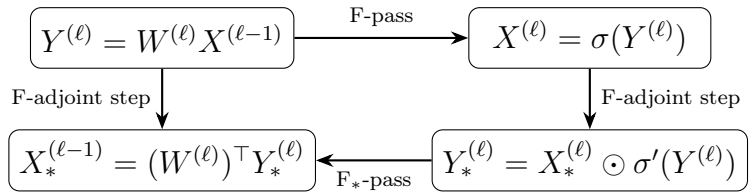
\begin{figure}[htp]
\centering
\begin{tikzpicture}[
node distance=1.5cm and 2.5cm,
A/.style={draw, rounded corners, minimum width=3.2cm,
minimum height=0.8cm, align=center},
every edge/.style={->, >=Stealth, draw, thick}
]
\node[A] (Y)     {$Y^{(\ell)} = W^{(\ell)} X^{(\ell-1)}$};
\node[A, below=0.8cm of Y]  (Xprev)
{$X_*^{(\ell-1)}=(W^{(\ell)})\transp Y_*^{(\ell)}$};
\node[A, right=2.3cm of Y]    (X)
{$X^{(\ell)}=\sigma(Y^{(\ell)})$};
\node[A, below=0.8cm of X]  (Ystar)
{$Y_*^{(\ell)}=X_*^{(\ell)}\had \sigma'(Y^{(\ell)})$};
\draw (Y)     edge node[midway,left,swap]{\scriptsize F-adjoint step} (Xprev);
\draw (X)     edge node[midway,right]{\scriptsize F-adjoint step} (Ystar);
\draw (Y.east) edge node[above,bend left=15]{\scriptsize F-pass} (X.west);
\draw (Ystar.west) edge node[below,bend left=15]{\scriptsize $\mathrm{F}_*$-pass}
(Xprev.east);
\end{tikzpicture}
\caption{Duality between F-propagation and F-adjoint at a single layer $\ell$.}
\label{fig:Fpasses}
\end{figure}
\begin{table}[h]
\centering
\renewcommand{\arraystretch}{1.4}
\caption{Structural duality between F-propagation and F-adjoint.}
\label{tab:fw-bw-duality}
\begin{tabular}{@{}lll@{}}
\toprule
\textbf{Feature}
& \textbf{F-propagation} $\FF$
& \textbf{F-adjoint} $\FFA$ \\
\midrule
Direction       & $\ell=1\to L$    & $\ell=L\to 1$  \\
State sequence  &
$(X^{(0)},Y^{(1)},X^{(1)},\ldots,Y^{(L)},X^{(L)})$  &
$(X^{(L)}_*,Y^{(L)}_*,X^{(L-1)}_*,\ldots,Y^{(1)}_*,X^{(0)}_*)$ \\
Affine step     &
$Y^{(\ell)}=W^{(\ell)}X^{(\ell-1)}$  &
$X^{(\ell-1)}_*=(W^{(\ell)})\transp Y^{(\ell)}_*$ \\
Nonlinear step  &
$X^{(\ell)}=\sigma(Y^{(\ell)})$  &
$Y^{(\ell)}_*=X^{(\ell)}_*\had\sigma'(Y^{(\ell)})$ \\
Initialisation  & $X^{(0)}=x$     & $X^{(L)}_*=\partial\loss/\partial X^{(L)}$ \\
Weight update   & ---         & $W^{(\ell)}_*=Y^{(\ell)}_*(X^{(\ell-1)})\transp$ \\
\bottomrule
\end{tabular}
\end{table}
\begin{theorem}[Backpropagation is the F-adjoint of the F-propagation
{\cite{boughammoura2023,boughammoura2024learning}}]
\label{thm:backprop-fadjoint}
Let $\NN[N_0,\ldots,N_L]$, $\FF$, and $\loss(f(x),y)$ be as in
\Cref{def:fprop}. Set $X^{(L)}_*=\partial\loss/\partial
X^{(L)}$ and let $\FFA$ be the F-adjoint with this seed. Then for
all $\ell=1,\ldots,L$,
\begin{equation}\label{eq:main-result}
\pp{\loss(f(x),y)}{W^{(\ell)}}
= Y^{(\ell)}_*\,(X^{(\ell-1)})\transp.
\end{equation}
In other words, the backpropagation pass associated with $\loss$ is
completely determined by the F-adjoint $\FFA$ of the F-propagation
$\FF$.
\end{theorem}
\begin{proof}
By \Cref{def:fadjoint}, the F-adjoint satisfies the same recurrence
as the classical backpropagation system of \Cref{thm:twostep} with
the same initialisation. Hence they coincide.
\end{proof}
\begin{corollary}[F-adjoint and classical delta errors]
\label{cor:delta-fadjoint}
Under the hypotheses of \Cref{thm:backprop-fadjoint}, the quantities
$Y^{(\ell)}_*$ in the F-adjoint coincide with the classical delta
errors $\delta^{(\ell)}$ of the generalised delta rule: $\delta^{(\ell)}
= Y^{(\ell)}_*$ for all $\ell=1,\ldots,L$.
\end{corollary}
\subsection{Global operator formulation}\label{subsec:global}
We now lift the layerwise F-adjoint recursion to a single global linear
equation. The global system below covers layers $\ell=1,\ldots,L$; the
input-layer adjoint $X_*^{(0)} = (W^{(1)})^T Y_*^{(1)}$ requires a separate
multiplication and is not part of the global linear system.
\begin{definition}[Stacked adjoint vectors]\label{def:stacked}
Stack the layerwise adjoints into global vectors
\[
\Xs = (X_{*}^{(1)}, X_{*}^{(2)}, \dots, X_{*}^{(L)})\transp, \qquad
\Ys = (Y_{*}^{(1)}, Y_{*}^{(2)}, \dots, Y_{*}^{(L)})\transp,
\]
and similarly $\X=(X^{(1)},\ldots,X^{(L)})\transp$,
$\Y=(Y^{(1)},\ldots,Y^{(L)})\transp$ for the forward states. Here the
notation $(\cdot,\ldots,\cdot)\transp$ denotes the block-column vector
obtained by vertically stacking the layerwise vectors $X_*^{(\ell)}\in
\R^{N_\ell}$ (not a transposed row of scalars); e.g.\ $\Xs\in\R^N$ with
$\Xs\big|_{\text{block }\ell}=X_*^{(\ell)}$. Let
$N=\sum_{\ell=1}^L N_\ell$ denote the total number of neurons (excluding
the input layer).
\end{definition}
\begin{definition}[Block-diagonal activation Jacobian]\label{def:sigmap}
For each layer $\ell$, define the diagonal matrix of activation derivatives
\begin{equation}\label{eq:Dell-def}
\Dell_\ell := \diag\bigl(\sigma'(Y^{(\ell)})\bigr) \in \R^{N_\ell \times N_\ell}.
\end{equation}
The global activation Jacobian is the block-diagonal matrix
\begin{equation}\label{eq:sigma_prime}
\Sigmap = \diag(\Dell_1, \Dell_2, \dots, \Dell_L).
\end{equation}
With this notation the pointwise chain rule~\eqref{eq:back1} becomes the global relation
\begin{equation}\label{eq:global_activation_adjoint}
\Ys = \Sigmap \Xs.
\end{equation}
\end{definition}
\begin{definition}[Global weight matrix and backward operator]\label{def:globalW}
The global weight matrix $W\in\R^{N\times N}$ is block strictly
lower-triangular, with $W^{(\ell)}$ positioned at block
$(\ell,\ell-1)$ for $\ell=2,\ldots,L$ and zeros elsewhere:
\[
W = \begin{pmatrix}
0 & 0 & \cdots & 0 & 0 \\
W^{(2)} & 0 & \cdots & 0 & 0 \\
0 & W^{(3)} & \cdots & 0 & 0 \\
\vdots & \vdots & \ddots & \vdots & \vdots \\
0 & 0 & \cdots & W^{(L)} & 0
\end{pmatrix}.
\]
Note that $W^{(1)}$ is explicitly excluded from this construction,
reflecting the fact that the F-adjoint system computes $X_*^{(1)}$
from $Y_*^{(1)}$ but does not need $W^{(1)}$ to do so; $X_*^{(0)}$
requires a separate multiplication by $(W^{(1)})^T$.
The \emph{global backward propagation operator} is
\begin{equation}\label{eq:B_operator}
\cB = W\transp \Sigmap.
\end{equation}
\end{definition}
\begin{proposition}[Structure of $\cB$]
\label{prop:structureB}
The operator $\cB$ is block strictly upper-triangular, with nonzero
blocks only on the first super-diagonal:
\[
\cB_{\ell,\ell+1} = (W^{(\ell+1)})\transp \Dell_{\ell+1},
\qquad \ell = 1, \dots, L-1.
\]
All other blocks are zero.
\end{proposition}
\begin{proof}
Since $W$ is block strictly lower-triangular, $W\transp$ is block
strictly upper-triangular with $(W^{(\ell)})\transp$ at position
$(\ell-1,\ell)$. Multiplication on the left by the block-diagonal
$\Sigmap$ preserves the block upper-triangular structure. Hence the
only nonzero blocks are at position $(\ell, \ell+1)$, equal to
$(W^{(\ell+1)})^T \Dell_{\ell+1}$.
\end{proof}
\begin{definition}[Injected source vector]\label{def:source}
The source vector $\bG\in\R^N$ injects the terminal gradient only at
the output layer,
\begin{equation}\label{eq:source_vector}
\bG = (0, 0, \dots, 0, \nabla_{X^{(L)}} \loss)\transp,
\end{equation}
where $\nabla_{X^{(L)}} \loss := \partial\loss/\partial X^{(L)} =
X_{*}^{(L)} \in \R^{N_L}$ is exactly the seed of
\eqref{eq:init-backward}; we use the two notations
$\nabla_{X^{(L)}}\loss$ and $X_*^{(L)}$ interchangeably throughout.
\end{definition}
\begin{theorem}[Nilpotency of $\cB$]
\label{thm:nilpotency}
The global backward propagation operator $\cB$ is nilpotent of
index at most $L$: $\cB^{L}=0$. Moreover, $\cB^{L-1} \neq 0$
precisely when the product
\[
\prod_{j=1}^{L-1} (W^{(j+1)})^T \Dell_{j+1}
\]
is nonzero. A sufficient (but not necessary) condition for this is
that every factor $(W^{(j+1)})^T \Dell_{j+1}$ is an invertible
square matrix (which requires $N_1=\cdots=N_L$ and every
$W^{(j+1)}$ invertible with every diagonal entry of
$\Dell_{j+1}$ nonzero): a product of invertible matrices is
invertible, hence nonzero. This invertibility condition is generic
(it fails only on a measure-zero set of weight matrices, for fixed
nonvanishing activation derivatives). Note that entrywise
nonvanishing of the individual factors is \emph{not} by itself
sufficient: two nonzero matrices can multiply to zero, e.g.\
$\begin{psmallmatrix}1&0\\0&0\end{psmallmatrix}
\begin{psmallmatrix}0&0\\0&1\end{psmallmatrix}=0$, so no purely
entrywise criterion can characterise nonvanishing of the product in
general.
\end{theorem}
\begin{proof}
The proof proceeds by analyzing the block structure of $\cB$ and establishing the index constraints for its powers via induction. We explicitly address the algebraic properties of matrix multiplication to ensure rigorous handling of zero divisors and summation cancellations.

\textbf{Step 1: Block structure of $\cB$.} 
By \Cref{prop:structureB}, the operator $\cB = W\transp \Sigmap$ is block strictly upper-triangular. Specifically, its only nonzero blocks lie on the first block super-diagonal:
\begin{equation}
    \cB_{i,m} = 
    \begin{cases} 
    (W^{(m)})\transp \Dell_{m} & \text{if } m = i + 1, \\
    0 & \text{otherwise},
    \end{cases}
\end{equation}
for block indices $i, m \in \{1, \dots, L\}$.

\textbf{Step 2: The single-term collapse of block multiplication.}
Consider the block multiplication formula for the $(i,j)$-th block of $\cB^{k+1}$:
\begin{equation}
    (\cB^{k+1})_{i,j} = \sum_{m=1}^L \cB_{i,m} (\cB^{k})_{m,j}.
\end{equation}
Because $\cB$ is strictly block-bidiagonal, $\cB_{i,m}$ is the exact zero matrix for all $m \neq i+1$. Substituting this into the summation yields:
\begin{equation}
    (\cB^{k+1})_{i,j} = 
    \underbrace{0 \cdot (\cB^{k})_{1,j}}_{m=1} + \dots + 
    \underbrace{\cB_{i,i+1} (\cB^{k})_{i+1,j}}_{m=i+1} + \dots + 
    \underbrace{0 \cdot (\cB^{k})_{L,j}}_{m=L}.
\end{equation}
Since multiplying the zero matrix by any matrix yields the zero matrix, every term in the sum where $m \neq i+1$ vanishes identically. The sum collapses to an \textit{exact equality} with a single surviving term:
\begin{equation}
    (\cB^{k+1})_{i,j} = \cB_{i,i+1} (\cB^{k})_{i+1,j}.
\end{equation}
\textit{Crucial Consequence:} Because the summation contains exactly one nonzero addend, there are no competing terms to add or subtract. Therefore, algebraic cancellation (e.g., $A + (-A) = 0$) is structurally impossible. The block $(\cB^{k+1})_{i,j}$ is nonzero \textit{if and only if} the single product $\cB_{i,i+1} (\cB^{k})_{i+1,j}$ is nonzero.

\textbf{Step 3: Inductive hypothesis and the handling of zero divisors.}
We claim that for any integer $k \ge 1$, a necessary condition for $(\cB^k)_{i,j} \neq 0$ is $j = i + k$. 
\textit{Base case ($k=1$):} This holds trivially by Step 1.
\textit{Inductive step:} Assume the claim holds for some $k \ge 1$. We must show it holds for $k+1$. 
Assume $(\cB^{k+1})_{i,j} \neq 0$. By the exact equality established in Step 2, this implies:
\begin{equation}
    \cB_{i,i+1} (\cB^{k})_{i+1,j} \neq 0.
\end{equation}
In the ring of matrices, if a product $AB \neq 0$, it is a universal necessity that both $A \neq 0$ and $B \neq 0$ (if either were the zero matrix, the product would be zero). Therefore, we must have:
\begin{equation}
    \cB_{i,i+1} \neq 0 \quad \text{AND} \quad (\cB^{k})_{i+1,j} \neq 0.
\end{equation}
\textit{Note on Zero Divisors:} We are \textit{not} using the false converse that nonzero factors guarantee a nonzero product. We are strictly using the valid implication $AB \neq 0 \implies A \neq 0 \land B \neq 0$. 
The first conjunct, $\cB_{i,i+1}\neq0$, is exactly the condition just derived from the hypothesis $(\cB^{k+1})_{i,j}\neq0$ and requires no further justification. The second conjunct, together with the inductive hypothesis, means $(\cB^{k})_{i+1,j} \neq 0$ requires $j = (i+1) + k = i + k + 1$. Thus, the claim holds for $k+1$.

\textbf{Step 4: Proof of $\cB^L = 0$.}
For $\cB^L$, the necessary condition for a nonzero block is $j = i + L$. However, the block indices are constrained to the finite set $1 \le i$ and $j \le L$. This implies $i + L \le L \implies i \le 0$, which contradicts $i \ge 1$. Therefore, no valid pair of indices $(i,j)$ exists to satisfy the condition. Consequently, every block of $\cB^L$ is identically zero, proving $\cB^L = 0$.

\textbf{Step 5: Condition for $\cB^{L-1} \neq 0$.}
For $k = L - 1$, the condition $j = i + L - 1$ subject to $1 \le i$ and $j \le L$ has the unique integer solution $i = 1$ and $j = L$. Thus, the top-right block $(1, L)$ is the \textit{only} block in $\cB^{L-1}$ that can possibly be nonzero. 
Unwinding the single-term recursion from Step 2 for this specific block yields:
\begin{align}
    (\cB^{L-1})_{1,L} &= \cB_{1,2} (\cB^{L-2})_{2,L} \nonumber \\
    &= \cB_{1,2} \cB_{2,3} (\cB^{L-3})_{3,L} \nonumber \\
    &= \dots \nonumber \\
    &= \cB_{1,2} \cB_{2,3} \cdots \cB_{L-1,L} \nonumber \\
    &= \prod_{j=1}^{L-1} (W^{(j+1)})\transp \Dell_{j+1}.
\end{align}
Because all other blocks of $\cB^{L-1}$ are identically zero, the entire matrix $\cB^{L-1}$ is nonzero \textit{if and only if} this specific matrix product is nonzero.
\end{proof}
\begin{corollary}[Invertibility of $I-\cB$]\label{cor:invertible}
Since $\cB$ is nilpotent, $I-\cB$ is invertible with inverse given
by the finite Neumann series
\begin{equation}\label{eq:neumann_inverse}
(I - \cB)^{-1} = \sum_{k=0}^{L-1} \cB^{k}.
\end{equation}
\end{corollary}
\begin{theorem}[Global fixed-point formulation]
\label{thm:global_fixed_point}
The F-adjoint variables satisfy the global fixed-point equation
\begin{equation}\label{eq:global_fixed_point}
(I - \cB) \Xs = \bG.
\end{equation}
Conversely, if $\Xs$ satisfies \eqref{eq:global_fixed_point} with
$\bG$ defined by~\eqref{eq:source_vector}, then the layerwise
F-adjoint recursions~\eqref{eq:fadjoint-rec} hold.
\end{theorem}
\begin{proof}
The proof proceeds in two parts: first, we show that the layerwise F-adjoint recursions imply the global fixed-point equation; second, we show the converse by unpacking the global equation block-by-block.

\textbf{Step 1: Define the stacked vectors and global matrices.}
By \Cref{def:stacked}, the adjoint variables are stacked into global vectors:
\begin{equation}
    \Xs = (X_{*}^{(1)}, X_{*}^{(2)}, \dots, X_{*}^{(L)})\transp, \qquad
    \Ys = (Y_{*}^{(1)}, Y_{*}^{(2)}, \dots, Y_{*}^{(L)})\transp.
\end{equation}
By \Cref{def:sigmap}, the global activation Jacobian is the block-diagonal matrix $\Sigmap = \diag(\Dell_1, \dots, \Dell_L)$, where $\Dell_\ell = \diag(\sigma'(Y^{(\ell)}))$. 
By \Cref{def:globalW}, the global weight matrix $W$ is block strictly lower-triangular with $W^{(\ell)}$ at block $(\ell, \ell-1)$. Consequently, $W\transp$ is block strictly upper-triangular with $(W^{(\ell)})\transp$ at block $(\ell-1, \ell)$. 
The global backward operator is $\cB = W\transp \Sigmap$, and by \Cref{def:source}, the source vector is $\bG = (0, \dots, 0, \nabla_{X^{(L)}} \loss)\transp$.

\textbf{Step 2: Translate the activation derivative step into global form.}
The layerwise F-adjoint recursion for the pre-activation gradient is $Y_{*}^{(\ell)} = X_{*}^{(\ell)} \had \sigma'(Y^{(\ell)})$ for $\ell=1,\dots,L$. 
Since $\Dell_\ell = \diag(\sigma'(Y^{(\ell)}))$, this element-wise multiplication is equivalent to the matrix-vector product $Y_{*}^{(\ell)} = \Dell_\ell X_{*}^{(\ell)}$. 
Stacking these equations for all $\ell=1,\dots,L$ yields the global relation:
\begin{equation}
    \Ys = \Sigmap \Xs.
\end{equation}

\textbf{Step 3: Translate the weight transpose step into global form.}
The layerwise recursion for the activation gradient is $X_{*}^{(\ell-1)} = (W^{(\ell)})\transp Y_{*}^{(\ell)}$ for $\ell=2,\dots,L$. 
Consider the block matrix-vector product $W\transp \Ys$. The $i$-th block of this product is $\sum_{j=1}^L (W\transp)_{i,j} Y_{*}^{(j)}$. 
Because $W\transp$ has nonzero blocks only at $(i, i+1)$, the sum collapses to a single term: $(W^{(i+1)})\transp Y_{*}^{(i+1)}$ for $i=1,\dots,L-1$, and $0$ for $i=L$. 
Thus, the vector $W\transp \Ys$ is exactly $(X_{*}^{(1)}, X_{*}^{(2)}, \dots, X_{*}^{(L-1)}, 0)\transp$.

\textbf{Step 4: Combine to derive the forward direction.}
We can express the full stacked vector $\Xs$ by adding the missing terminal component $X_{*}^{(L)}$ to the result of Step 3:
\begin{equation}
    \Xs = W\transp \Ys + (0, \dots, 0, X_{*}^{(L)})\transp.
\end{equation}
By the terminal initialisation of the backward pass, $X_{*}^{(L)} = \nabla_{X^{(L)}} \loss$, so the second term is exactly the source vector $\bG$. Substituting $\Ys = \Sigmap \Xs$ from Step 2, and recognising the operator $\cB = W\transp \Sigmap$, we obtain:
\begin{equation}
    \Xs = W\transp \Sigmap \Xs + \bG \implies \Xs = \cB \Xs + \bG.
\end{equation}
Rearranging terms yields the global fixed-point equation:
\begin{equation}
    (I - \cB) \Xs = \bG.
\end{equation}

\textbf{Step 5: Unpack the global equation block-by-block (Converse).}
Conversely, assume $\Xs$ satisfies $(I - \cB)\Xs = \bG$. 
The $i$-th block of this equation is $X_{*}^{(i)} - (\cB \Xs)_i = \bG_i$. 
By \Cref{prop:structureB}, $\cB$ has nonzero blocks only on the first super-diagonal, meaning $(\cB \Xs)_i = \cB_{i,i+1} X_{*}^{(i+1)} = (W^{(i+1)})\transp \Dell_{i+1} X_{*}^{(i+1)}$ for $i < L$, and $0$ for $i = L$.

\textbf{Step 6: Recover the terminal condition.}
For the final block $i = L$, the equation is $X_{*}^{(L)} - 0 = \bG_L$. 
By the definition of the source vector, $\bG_L = \nabla_{X^{(L)}} \loss$. 
Thus, $X_{*}^{(L)} = \nabla_{X^{(L)}} \loss$, which perfectly recovers the terminal initialisation of the backward pass.

\textbf{Step 7: Recover the layerwise recursions.}
For any block $i \in \{1, \dots, L-1\}$, the equation is:
\begin{equation}
    X_{*}^{(i)} - (W^{(i+1)})\transp \Dell_{i+1} X_{*}^{(i+1)} = 0 \implies X_{*}^{(i)} = (W^{(i+1)})\transp \left( \Dell_{i+1} X_{*}^{(i+1)} \right).
\end{equation}
Let $\ell = i+1$. As $i$ ranges from $1$ to $L-1$, $\ell$ ranges from $2$ to $L$. Substituting $\ell$ yields:
\begin{equation}
    X_{*}^{(\ell-1)} = (W^{(\ell)})\transp \left( \Dell_\ell X_{*}^{(\ell)} \right).
\end{equation}
Defining $Y_{*}^{(\ell)} := \Dell_\ell X_{*}^{(\ell)}$ (which is exactly $X_{*}^{(\ell)} \had \sigma'(Y^{(\ell)})$), we recover:
\begin{equation}
    Y_{*}^{(\ell)} = X_{*}^{(\ell)} \had \sigma'(Y^{(\ell)}), \qquad X_{*}^{(\ell-1)} = (W^{(\ell)})\transp Y_{*}^{(\ell)},
\end{equation}
for all $\ell = 2, \dots, L$. Together with the terminal condition from Step 6 and the definition of $Y_*^{(L)}$, this constitutes the complete set of layerwise F-adjoint recursions.
\end{proof}
\begin{corollary}[Neumann series expansion]\label{cor:neumann}
The solution to the global fixed-point equation is given by the
finite Neumann series
\begin{equation}\label{eq:neumann_series}
\Xs = \sum_{k=0}^{L-1} \cB^{k} \bG.
\end{equation}
\end{corollary}
\begin{corollary}[Equivalence of characterisations]\label{cor:equivalence}
Let $\cB=W\transp\Sigmap$ with $\cB^L=0$. For a candidate vector
$\Xs=(X_{*}^{(1)},\ldots,X_{*}^{(L)})\transp$ the following are
equivalent:
\begin{enumerate}[label=\textbf{(\Alph*)}]
\item \textbf{Layerwise.} The components satisfy the terminal
condition $X_{*}^{(L)}= \nabla_{X^{(L)}}\loss$ and the backward
recursions~\eqref{eq:back1}--\eqref{eq:back2};
\item \textbf{Global fixed point.} $(I-\cB)\Xs=\bG$;
\item \textbf{Neumann series.} $\Xs=\sum_{k=0}^{L-1}\cB^{k}\bG$.
\end{enumerate}
Moreover, since $I-\cB$ is invertible (\Cref{cor:invertible}), the
vector $\Xs$ satisfying (A)--(C) is unique, and coincides with the
F-adjoint and with backpropagation.
\end{corollary}
\begin{figure}[htbp]
\centering
\begin{tikzpicture}[
cell/.style={draw,minimum width=1.4cm,minimum height=1.0cm,font=\small,anchor=center},
zero/.style={cell,fill=gray!10},
diag/.style={cell,fill=blue!8},
super/.style={cell,fill=red!12},
brace/.style={decorate,decoration={brace,amplitude=6pt,mirror}},
]
\matrix (M) [matrix of nodes, nodes in empty cells,row sep=2pt, column sep=2pt] {
|[diag]|$I$          & |[super]|$-\cB_{12}$ & |[zero]|$0$ & |[zero]|$0$ \\
|[zero]|$0$        & |[diag]|$I$            & |[super]|$-\cB_{23}$ & |[zero]|$0$ \\
|[zero]|$0$        & |[zero]|$0$          & |[diag]|$I$   & |[super]|$-\cB_{34}$ \\
|[zero]|$0$        & |[zero]|$0$          & |[zero]|$0$ & |[diag]|$I$ \\
};
\draw[brace] (M-1-1.south west) -- (M-4-1.south west)
node[midway,left=8pt]{\small $I-\cB$};
\node[below=6pt of M,font=\footnotesize\itshape]
{Identity diagonal + strictly upper bidiagonal super-diagonal};
\end{tikzpicture}
\caption{Block structure of $I-\cB$ for $L=4$.}
\label{fig:bidiag}
\end{figure}
\subsection{Block back-substitution and path decomposition}\label{subsec:backsub}
\begin{proposition}[Back-substitution interpretation]
\label{prop:backsub}
Solving $(I-\cB)\Xs=\bG$ by standard block back-substitution on the
upper bidiagonal system reproduces the backward pass in exactly $L$
steps, visiting layers $L,L-1,\ldots,1$ in order.
\end{proposition}
\begin{proof}
Back-substitution on the upper-triangular system with $U_{\ell\ell}=I$
gives $X_{*}^{(L)} = \nabla_{X^{(L)}}\loss$ and
$X_{*}^{(\ell)} = \cB_{\ell,\ell+1}X_{*}^{(\ell+1)}$ for
$\ell=L-1,\ldots,1$, which is precisely the layerwise backward pass.
\end{proof}
\begin{proposition}[Path expansion]
\label{prop:path}
The $\ell$-th block of $\cB^{k}\bG$ is nonzero only when $k=L-\ell$,
and in that case
\[
\bigl(\cB^{L-\ell}\bG\bigr)_\ell
= \left[\prod_{j=\ell}^{L-1}(W^{(j+1)})\transp\Dell_{j+1}\right]
\nabla_{X^{(L)}}\loss.
\]
In particular, $X_{*}^{(\ell)} = \bigl(\cB^{L-\ell}\bG\bigr)_\ell$.
\end{proposition}
\begin{proof}
The source $\bG$ has nonzero content only in block $L$. Since $\cB$
is strictly upper bidiagonal, each multiplication by $\cB$ moves
content from block $\ell+1$ to block $\ell$. Hence $(\cB^k\bG)_\ell
\neq 0$ only for $\ell=L-k$. The product formula follows by
unwinding the recursion.
\end{proof}
\begin{remark}[Single-path collapse]\label{rem:collapse}
For strictly feedforward architectures, the Neumann series collapses
to a \emph{single nonzero term per block}: each layer $\ell$
receives its adjoint from exactly the $(L-\ell)$-th power of $\cB$.
The summation structure becomes essential only when $\cB$ is not
strictly bidiagonal, as with residual connections.
\end{remark}
\begin{lemma}[General path-expansion lemma]\label{lem:path-general}
Let $\cB$ be any block strictly upper-triangular nilpotent
operator on blocks $1,\ldots,L$ (not necessarily bidiagonal;
$\cB_{i,j}$ may be nonzero for any $j>i$, as when skip connections
are present), and let $\bG$ have nonzero content only in block $L$.
For $\ell<L$ define a \emph{path} from $L$ to $\ell$ as a strictly
decreasing chain of block indices $L=i_0>i_1>\cdots>i_m=\ell$
(equivalently, written increasing,
$\ell=i_m<i_{m-1}<\cdots<i_0=L$) with $m\ge1$ such that every
factor $\cB_{i_k,i_{k-1}}$ is nonzero, and let its \emph{weight} be
the ordered matrix product
$\cB_{i_m,i_{m-1}}\cB_{i_{m-1},i_{m-2}}\cdots\cB_{i_1,i_0}$. Then
\begin{equation}\label{eq:path-general}
    X_{*}^{(\ell)} = \sum_{m=1}^{L-\ell} \sum_{L > j_1 > \dots > j_{m-1} > \ell} \cB_{\ell, j_{m-1}} \cdots \cB_{j_1, L} \,\bG_L,
\end{equation}
the sum ranging over all paths $P$ from $L$ to $\ell$ in the sense
above (and, for $\ell=L$, $X_*^{(L)}=\bG_L$, the empty-path term).
\end{lemma}
\begin{proof}
By the definition of block matrix multiplication, the $\ell$-th block of $\cB^m \bG$ is given by
\begin{equation}
    (\cB^{m}\bG)_\ell = \sum_{i_1,\ldots,i_{m-1}=1}^L 
    \cB_{\ell,i_{m-1}}\cB_{i_{m-1},i_{m-2}}\cdots\cB_{i_1,L}\,\bG_L.
\end{equation}
Because $\cB$ is strictly block upper-triangular, any factor $\cB_{a,b}$ is identically zero unless $b > a$. Consequently, a term in the summation is nonzero if and only if the intermediate indices form a strictly increasing chain:
\begin{equation}
    \ell < i_{m-1} < i_{m-2} < \cdots < i_1 < L.
\end{equation}
Reversing the order of multiplication to match the forward flow of indices, the weight of such a valid chain is exactly the ordered matrix product $\cB_{\ell,i_{m-1}}\cB_{i_{m-1},i_{m-2}}\cdots\cB_{i_1,L}$. 
Summing over all possible path lengths $m=0, \ldots, L-1$ (noting that $\cB^L=0$ by nilpotency) and collecting all valid chains yields exactly the sum over all paths $P: L \to \ell$ as defined in the statement. For $m=0$, the chain is empty, yielding the identity operator on $\bG_L$, which gives $X_*^{(L)} = \bG_L$.
\end{proof}
\begin{remark}
\Cref{prop:path} is the special case in which $\cB$ is strictly
bidiagonal ($\cB_{i,j}=0$ unless $j=i+1$): there each chain
$\ell<i_{m-1}<\cdots<L$ is forced to be the single consecutive run
$\ell,\ell+1,\ldots,L$, so exactly one path exists and the sum
in~\eqref{eq:path-general} collapses to one term, recovering
\Cref{rem:collapse}. \Cref{lem:path-general} shows that this
single-path collapse is not a separate phenomenon but simply the
bidiagonal instance of the same general path-sum identity that
governs arbitrary DAG-structured backward operators, including
those induced by skip connections (\Cref{ex:residual}).
\end{remark}
\begin{example}[Complete path-expansion for $L=5$ with illustrative figures]\label{ex:L5-paths-final}
To fully illustrate \Cref{lem:path-general} in a generic setting with skip connections, we instantiate the global backward operator for a network of depth $L=5$. Unlike strictly feedforward networks, we allow off-bidiagonal blocks $\cB_{i,j}$ ($j > i+1$) to represent arbitrary skip connections, resulting in a dense upper-triangular structure.
\paragraph{Refined path-expansion equation.}
For $L=5$, the general path-expansion~\eqref{eq:path-general} takes the explicit form:
\begin{equation}\label{eq:path-L5}
X_{*}^{(\ell)} = \sum_{m=1}^{5-\ell} \sum_{5 > j_1 > \dots > j_{m-1} > \ell} 
\cB_{\ell, j_{m-1}} \cB_{j_{m-1}, j_{m-2}} \cdots \cB_{j_1, 5} \,\bG_5,
\end{equation}
where the inner sum ranges over all strictly decreasing sequences of $m-1$ intermediate layers connecting layer $5$ to layer $\ell$.
\paragraph{Block matrix structure.}
The global backward operator $\cB$ is a $5 \times 5$ block strictly upper-triangular matrix:
\begin{equation}
\cB = \begin{pmatrix}
0 & \cB_{1,2} & \cB_{1,3} & \cB_{1,4} & \cB_{1,5} \\
0 & 0 & \cB_{2,3} & \cB_{2,4} & \cB_{2,5} \\
0 & 0 & 0 & \cB_{3,4} & \cB_{3,5} \\
0 & 0 & 0 & 0 & \cB_{4,5} \\
0 & 0 & 0 & 0 & 0
\end{pmatrix}.
\end{equation}
\paragraph{Computational graph illustration.}
\Cref{fig:L5-paths-final} visualises the backward computational graph with gradient flowing from right to left. All edges are routed \emph{concentrically above the nodes} with curvature strictly increasing by hop length to prevent any visual crossings:
\begin{itemize}
    \item \textbf{Sequential (blue, solid):} Standard layer-to-layer connections ($j = i+1$), bend angle $15^\circ$.
    \item \textbf{Skip 1 (red, dashed):} Connections skipping one layer ($j = i+2$), bend angle $30^\circ$.
    \item \textbf{Skip 2 (green, dotted):} Connections skipping two layers ($j = i+3$), bend angle $45^\circ$.
    \item \textbf{Skip 3 (orange, dash-dotted):} Connections skipping three layers ($j = i+4$), bend angle $60^\circ$.
\end{itemize}
\begin{figure}[htbp]
\centering
\begin{tikzpicture}[
    node distance=2.5cm,
    layer/.style={draw, circle, minimum size=1.1cm, font=\bfseries\large, fill=blue!5, thick, draw=blue!70!black},
    feed/.style={->, >=Stealth, thick, blue!80!black},
    skip1/.style={->, >=Stealth, thick, dashed, red!80!black},
    skip2/.style={->, >=Stealth, thick, dotted, green!60!black},
    skip3/.style={->, >=Stealth, thick, dashdotted, orange!90!black},
    lbl/.style={font=\scriptsize, fill=white, inner sep=1.5pt, text=black, draw=none}
]
% Nodes (L1 on left, L5 on right for natural backward flow)
\node[layer] (L1) {$1$};
\node[layer, right=of L1] (L2) {$2$};
\node[layer, right=of L2] (L3) {$3$};
\node[layer, right=of L3] (L4) {$4$};
\node[layer, right=of L4] (L5) {$5$};
% Sequential (lowest arc, bend=15)
\draw[feed] (L5) to[bend right=15] node[lbl, above, sloped] {$\cB_{4,5}$} (L4);
\draw[feed] (L4) to[bend right=15] node[lbl, above, sloped] {$\cB_{3,4}$} (L3);
\draw[feed] (L3) to[bend right=15] node[lbl, above, sloped] {$\cB_{2,3}$} (L2);
\draw[feed] (L2) to[bend right=15] node[lbl, above, sloped] {$\cB_{1,2}$} (L1);
% 1-hop skips (second arc, bend=30)
\draw[skip1] (L5) to[bend right=30] node[lbl, above, sloped] {$\cB_{3,5}$} (L3);
\draw[skip1] (L4) to[bend right=30] node[lbl, above, sloped] {$\cB_{2,4}$} (L2);
\draw[skip1] (L3) to[bend right=30] node[lbl, above, sloped] {$\cB_{1,3}$} (L1);
% 2-hop skips (third arc, bend=45)
\draw[skip2] (L5) to[bend right=45] node[lbl, above, sloped] {$\cB_{2,5}$} (L2);
\draw[skip2] (L4) to[bend right=45] node[lbl, above, sloped] {$\cB_{1,4}$} (L1);
% 3-hop skip (highest arc, bend=60)
\draw[skip3] (L5) to[bend right=60] node[lbl, above, sloped] {$\cB_{1,5}$} (L1);
% Legend
\begin{scope}[shift={(0.5,-2.5)}]
    \draw[feed] (0,0) -- (0.8,0);
    \node[font=\scriptsize, anchor=west, text=black] at (1.0,0) {Sequential};
    \draw[skip1] (3.0,0) -- (3.8,0);
    \node[font=\scriptsize, anchor=west, text=black] at (4.0,0) {Skip 1};
    \draw[skip2] (6.0,0) -- (6.8,0);
    \node[font=\scriptsize, anchor=west, text=black] at (7.0,0) {Skip 2};
    \draw[skip3] (9.0,0) -- (9.8,0);
    \node[font=\scriptsize, anchor=west, text=black] at (10.0,0) {Skip 3};
\end{scope}
\end{tikzpicture}
\caption{Computational graph for the backward pass ($L=5$) with all possible skip connections. Gradient flows from right to left. All edges are routed concentrically above the nodes with curvature increasing by hop length ($15^\circ$, $30^\circ$, $45^\circ$, $60^\circ$), ensuring zero crossings and perfect visual separation of the $15$ distinct paths.}
\label{fig:L5-paths-final}
\end{figure}
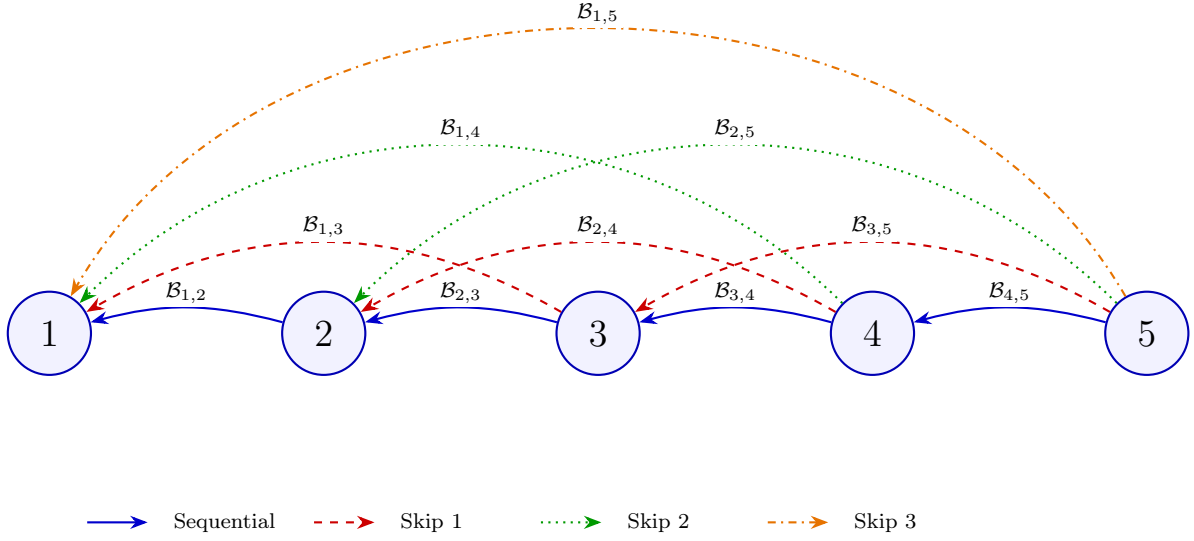
\paragraph{Explicit path equations.}
Applying \Cref{lem:path-general}, the adjoint variable $X_{*}^{(\ell)}$ is the sum of the weights of all valid paths from layer $5$ to layer $\ell$. We group terms by hop count $m$:
\begin{align}
X_{*}^{(5)} &= \bG_5, \tag{empty path, $m=0$, 1 path} \\[6pt]
X_{*}^{(4)} &= \underbrace{\cB_{4,5}\,\bG_5}_{\text{seq. } 5\to4}, \tag{$m=1$, 1 path} \\[6pt]
X_{*}^{(3)} &= \underbrace{\cB_{3,5}\,\bG_5}_{\text{skip 1}} 
+ \underbrace{\cB_{3,4}\cB_{4,5}\,\bG_5}_{\text{seq. } 5\to4\to3}, \tag{$m=1,2$, 2 paths} \\[6pt]
X_{*}^{(2)} &= \underbrace{\cB_{2,5}\,\bG_5}_{\text{skip 2}} 
+ \underbrace{\cB_{2,4}\cB_{4,5}\,\bG_5 + \cB_{2,3}\cB_{3,5}\,\bG_5}_{\text{skip 1}} 
+ \underbrace{\cB_{2,3}\cB_{3,4}\cB_{4,5}\,\bG_5}_{\text{seq.}}, \tag{$m=1,2,3$, 4 paths} \\[6pt]
X_{*}^{(1)} &= \underbrace{\cB_{1,5}\,\bG_5}_{\text{skip 3}} 
+ \underbrace{\cB_{1,4}\cB_{4,5}\,\bG_5 + \cB_{1,3}\cB_{3,5}\,\bG_5 + \cB_{1,2}\cB_{2,5}\,\bG_5}_{\text{skip 2}} \nonumber \\
&\quad + \underbrace{\cB_{1,3}\cB_{3,4}\cB_{4,5}\,\bG_5 + \cB_{1,2}\cB_{2,4}\cB_{4,5}\,\bG_5 + \cB_{1,2}\cB_{2,3}\cB_{3,5}\,\bG_5}_{\text{skip 1}} 
+ \underbrace{\cB_{1,2}\cB_{2,3}\cB_{3,4}\cB_{4,5}\,\bG_5}_{\text{seq.}}. \tag{$m=1,2,3,4$, 8 paths}
\end{align}
\paragraph{Combinatorial structure.}
The number of distinct paths from layer $5$ to layer $\ell$ is exactly $2^{5-\ell-1}$ for $\ell < 5$:
\begin{itemize}
    \item Layer $4$: $2^0 = 1$ path
    \item Layer $3$: $2^1 = 2$ paths
    \item Layer $2$: $2^2 = 4$ paths
    \item Layer $1$: $2^3 = 8$ paths
\end{itemize}
This demonstrates the \emph{combinatorial explosion} of the Neumann series when skip connections are fully dense: layer $1$ receives gradient information through $8$ distinct computational paths, as visualised in \Cref{fig:L5-paths-final}.
\paragraph{Single-path collapse for strictly feedforward networks.}
If the network is strictly feedforward (no skip connections), all off-bidiagonal blocks $\cB_{i,j}$ ($j > i+1$) are identically zero. The equations collapse to:
\begin{align}
X_{*}^{(4)} &= \cB_{4,5}\,\bG_5, \\
X_{*}^{(3)} &= \cB_{3,4}\cB_{4,5}\,\bG_5, \\
X_{*}^{(2)} &= \cB_{2,3}\cB_{3,4}\cB_{4,5}\,\bG_5, \\
X_{*}^{(1)} &= \cB_{1,2}\cB_{2,3}\cB_{3,4}\cB_{4,5}\,\bG_5,
\end{align}
recovering the \emph{single-path collapse} of \Cref{rem:collapse}: each layer receives its adjoint from exactly one term in the Neumann series, corresponding to the unique sequential path. This is the special case where the general path-sum~\eqref{eq:path-L5} reduces to a single term for each $\ell$.
\end{example}
\subsection{Global theory of F-symmetry}\label{subsec:fsym}
\begin{definition}[F-symmetric network]
\label{def:fsym-global}
A network $\NN[N_0,\ldots,N_L]$ is \emph{F-symmetric} (at a given input $X^{(0)}$ and seed $X_{*}^{(L)}$) if the F-adjoint activation sequence coincides with the F-propagation activation sequence, i.e.,
\begin{equation}\label{eq:fsym-def}
\Xs = \X.
\end{equation}
\end{definition}
\begin{remark}[The boundary term $X_{*}^{(0)}$]\label{rem:layer-zero}
The stacked vectors $\Xs$ and $\X$ are built from layers $1,\ldots,L$ only. The definition of F-symmetry compares $\FF$ and $\FFA$ only on the indices $\ell=1,\ldots,L$ that both sequences share; it deliberately says nothing about $X^{(0)}$ versus $X_{*}^{(0)}$.
\end{remark}
\begin{proposition}[Fixed-point characterisation of F-symmetry]
\label{prop:fsym-fixedpoint-char}
For a fixed seed $X_{*}^{(L)}=X^{(L)}$, the network is F-symmetric if and only if the forward activations $\X$ satisfy the global fixed-point equation:
\begin{equation}\label{eq:fsym-operator}
(I-\cB)\X = \bG, \qquad \bG = (0,\ldots,0,X^{(L)})\transp.
\end{equation}
\end{proposition}
\begin{proof}
By \Cref{thm:global_fixed_point}, $\Xs$ always satisfies $(I-\cB)\Xs=\bG$. By \Cref{cor:equivalence}, $I-\cB$ is invertible, so this is the unique solution. Hence $\Xs=\X$ if and only if $\X$ satisfies the equation.
\end{proof}
\begin{proposition}[Orthogonal networks are F-symmetric]
\label{prop:fsym-orthogonal-global}
Suppose $\sigma=\mathrm{Id}$, $N_1=\cdots=N_L=:N$, and each $W^{(\ell)}\in\R^{N\times N}$, $\ell=2,\ldots,L$, is orthogonal, $(W^{(\ell)})\transp W^{(\ell)}=I_N$. Set the seed $X_{*}^{(L)}:=X^{(L)}$. Then the network is F-symmetric.
\end{proposition}
\begin{proof}
We show $X_{*}^{(\ell)}=X^{(\ell)}$ by downward induction. Base case: $X_{*}^{(L)}=X^{(L)}$ by choice of seed. Inductive step: if $X_{*}^{(\ell)}=X^{(\ell)}$, then $X_{*}^{(\ell-1)} = (W^{(\ell)})^T X_{*}^{(\ell)} = (W^{(\ell)})^T X^{(\ell)} = (W^{(\ell)})^T W^{(\ell)} X^{(\ell-1)} = X^{(\ell-1)}$ by orthogonality. Thus $\Xs=\X$.
\end{proof}
\begin{proposition}[Spectral characterisation of F-symmetry]
\label{prop:fsym-spectral}
Suppose $\sigma=\mathrm{Id}$, $N_1=\cdots=N_L=:N$, and every input $X^{(0)}$ is \emph{admissible}, meaning $X^{(\ell-1)}$ ranges over all of $\R^N$ for each $\ell=2,\ldots,L$. Fix the seed $X_{*}^{(L)}=X^{(L)}$. Then the network is F-symmetric if and only if $(W^{(\ell)})\transp W^{(\ell)}=I_N$ for every $\ell=2,\ldots,L$.
\end{proposition}
\begin{proof}
($\Leftarrow$) This is \Cref{prop:fsym-orthogonal-global}.
($\Rightarrow$) F-symmetry forces $(I-\cB)\X = \bG$, which block-by-block means $X^{(\ell-1)} = (W^{(\ell)})^T X^{(\ell)} = (W^{(\ell)})^T W^{(\ell)} X^{(\ell-1)}$ for every attainable $X^{(\ell-1)}$. By admissibility, this holds for all vectors in $\R^N$, so $(W^{(\ell)})^T W^{(\ell)}=I_N$.
\end{proof}
\begin{proposition}[Layerwise condition for F-symmetry, general activation]
\label{prop:gen-fsym}
Let $\sigma$ be an arbitrary coordinate-wise differentiable activation. Fix an input $X^{(0)}$, let $\FF$ be the resulting F-propagation, and fix the seed $X_{*}^{(L)}:=X^{(L)}$. Then the network is F-symmetric \emph{at this input} if and only if, for every $\ell=1,\ldots,L-1$,
\begin{equation}\label{eq:gen-fsym-weight}
\sigma\bigl(Y^{(\ell)}\bigr) = (W^{(\ell+1)})\transp \Bigl( \sigma\bigl(Y^{(\ell+1)}\bigr) \had \sigma'\bigl(Y^{(\ell+1)}\bigr) \Bigr)
\tag{C$_\ell$}
\end{equation}
holds at the realised pre-activations.
\end{proposition}
\begin{proof}
By \Cref{prop:fsym-fixedpoint-char}, F-symmetry is equivalent to $(I-\cB)\X = \bG$. Unpacking this block equation for $\ell = 1, \dots, L-1$ yields $X^{(\ell)} = \cB_{\ell, \ell+1} X^{(\ell+1)}$. Substituting the definitions $X^{(\ell)} = \sigma(Y^{(\ell)})$, $\cB_{\ell, \ell+1} = (W^{(\ell+1)})^T \Dell_{\ell+1}$, and $\Dell_{\ell+1} X^{(\ell+1)} = \diag(\sigma'(Y^{(\ell+1)})) \sigma(Y^{(\ell+1)}) = \sigma(Y^{(\ell+1)}) \had \sigma'(Y^{(\ell+1)})$, we obtain exactly condition \eqref{eq:gen-fsym-weight}. The block equation for $\ell=L$ is $X^{(L)} = X^{(L)}$, which is satisfied by the choice of seed.
\end{proof}
\begin{remark}[Decoupling of the $X$ and $Y$ sequences]
\label{rem:gen-fsym-scope}
Under \Cref{def:fsym-global}, F-symmetry strictly requires only the coincidence of the post-activation sequences ($\Xs = \X$). It deliberately does not require the pre-activation sequences to coincide ($\Ys = \Y$). Indeed, $\Ys = \Y$ would impose the much stronger pointwise condition $\sigma(Y_i)\sigma'(Y_i) = Y_i$ at every coordinate, which forces $\sigma = \pm \mathrm{Id}$. By restricting F-symmetry to $\Xs = \X$, we obtain a meaningful, non-vacuous concept for arbitrary nonlinear activations $\sigma$, characterised cleanly by the layerwise consistency condition \eqref{eq:gen-fsym-weight}.
\end{remark}
% ---------------------------------------------------------------
\section{Illustrative Examples}\label{sec:results}
% ---------------------------------------------------------------
This section makes every construction of \Cref{sec:methodology} concrete
through four progressively complex worked examples. We begin with a
complete small network that illustrates the full global machinery
(\Cref{ex:global-structure}), then examine the scalar case to expose the
single-path collapse phenomenon (\Cref{ex:scalar}), verify the framework
for non-square weight matrices (\Cref{ex:vector1}), and conclude with a
comparative example showing how residual connections fundamentally alter
the path structure (\Cref{ex:residual}).
\begin{example}[Complete global system for a small network]
\label{ex:global-structure}
This example walks through the entire global-operator construction
for a concrete $L=3$ network, demonstrating the equivalence between
the Neumann series, direct inversion, and classical layerwise
backpropagation.
\paragraph{Network configuration.}
Let $N_1=2$, $N_2=2$, $N_3=1$ (a scalar-output regression network),
with weight matrices
\[
W^{(2)} = \begin{pmatrix} 2 & 0 \\ 1 & 1 \end{pmatrix}, \qquad
W^{(3)} = \begin{pmatrix} 1 & -1 \end{pmatrix}.
\]
At the current input, suppose the activation derivatives are
\[
\sigma'(Y^{(2)}) = (0.5,\; 1), \qquad \sigma'(Y^{(3)}) = 2,
\]
so that
\[
\Dell_2 = \begin{pmatrix} 0.5 & 0 \\ 0 & 1 \end{pmatrix}, \qquad
\Dell_3 = \begin{pmatrix} 2 \end{pmatrix}.
\]
Take the loss gradient at the output to be $g = \nabla_{X^{(3)}}\loss = 1$.
\paragraph{Step 1: Assemble the backward blocks.}
By \Cref{prop:structureB},
\[
\cB_{1,2} = (W^{(2)})^T\Dell_2
= \begin{pmatrix} 2 & 1 \\ 0 & 1 \end{pmatrix}
\begin{pmatrix} 0.5 & 0 \\ 0 & 1 \end{pmatrix}
= \begin{pmatrix} 1 & 1 \\ 0 & 1 \end{pmatrix},
\]
\[
\cB_{2,3} = (W^{(3)})^T\Dell_3
= \begin{pmatrix} 1 \\ -1 \end{pmatrix}(2)
= \begin{pmatrix} 2 \\ -2 \end{pmatrix}.
\]
\paragraph{Step 2: The global linear system.}
With $N = N_1+N_2+N_3 = 5$, the full system $(I-\cB)\Xs = \bG$ is
\[
\left(\begin{array}{cc|cc|c}
1 & 0 & -1 & -1 & 0 \\
0 & 1 & 0 & -1 & 0 \\ \hline
0 & 0 & 1 & 0 & -2 \\
0 & 0 & 0 & 1 & 2 \\ \hline
0 & 0 & 0 & 0 & 1
\end{array}\right)
\begin{pmatrix} X_{*,1}^{(1)} \\ X_{*,2}^{(1)} \\ X_{*,1}^{(2)} \\ X_{*,2}^{(2)} \\ X_{*}^{(3)} \end{pmatrix}
=
\begin{pmatrix} 0 \\ 0 \\ 0 \\ 0 \\ 1 \end{pmatrix}.
\]
Reading this bottom-to-top is precisely block back-substitution
(\Cref{prop:backsub}), yielding
\[
X_*^{(3)} = 1,\quad
X_*^{(2)} = \begin{pmatrix}2 \\ -2\end{pmatrix},\quad
X_*^{(1)} = \begin{pmatrix}0 \\ -2\end{pmatrix}.
\]
\paragraph{Step 3: Neumann series verification.}
Since $\cB^3=0$, the Neumann series has only three terms:
\[
\cB^0\bG = \begin{pmatrix}0\\0\\0\\0\\1\end{pmatrix},\quad
\cB^1\bG = \begin{pmatrix}0\\0\\2\\-2\\0\end{pmatrix},\quad
\cB^2\bG = \begin{pmatrix}0\\-2\\0\\0\\0\end{pmatrix}.
\]
Summing gives $\Xs = (0,-2,2,-2,1)^T$, matching the back-substitution
result. Direct inversion of $I-\cB$ gives the same vector, confirming
\Cref{cor:equivalence}.
\paragraph{Step 4: Consistency with layerwise backpropagation.}
Running the standard two-step recursion:
\[
X_*^{(3)} = 1,\quad
Y_*^{(3)} = 2,\quad
X_*^{(2)} = (W^{(3)})^T Y_*^{(3)} = \begin{pmatrix}2 \\ -2\end{pmatrix},
\]
\[
Y_*^{(2)} = \Dell_2 X_*^{(2)} = \begin{pmatrix}1 \\ -2\end{pmatrix},\quad
X_*^{(1)} = (W^{(2)})^T Y_*^{(2)} = \begin{pmatrix}0 \\ -2\end{pmatrix}.
\]
Stacking yields exactly the same $\Xs$, demonstrating the three-way
equivalence of the framework.
\end{example}
\begin{example}[Scalar network: single-path collapse]
\label{ex:scalar}
This example isolates the core mechanism of the framework by
considering scalar weights and activations, where the block structure
reduces to ordinary matrix multiplication.
Let $L=3$ with scalar weights $w^{(2)}, w^{(3)}$ and scalar
activation derivatives $s^{(2)} = \sigma'(Y^{(2)})$,
$s^{(3)} = \sigma'(Y^{(3)})$. The backward operator and source are
\[
\cB = \begin{pmatrix}
0 & w^{(2)}s^{(2)} & 0 \\
0 & 0 & w^{(3)}s^{(3)} \\
0 & 0 & 0
\end{pmatrix}, \qquad
\bG = \begin{pmatrix} 0 \\ 0 \\ g \end{pmatrix},
\quad g = \nabla_{X^{(3)}}\loss.
\]
The Neumann series yields
\[
\cB^0\bG = \begin{pmatrix}0\\0\\g\end{pmatrix},\quad
\cB^1\bG = \begin{pmatrix}0\\ w^{(3)}s^{(3)}g \\0\end{pmatrix},\quad
\cB^2\bG = \begin{pmatrix}w^{(2)}s^{(2)}w^{(3)}s^{(3)}g\\0\\0\end{pmatrix},
\]
so
\[
\Xs =
\begin{pmatrix}
w^{(2)}w^{(3)}s^{(2)}s^{(3)}g \\
w^{(3)}s^{(3)}g \\
g
\end{pmatrix}.
\]
\paragraph{Key observation.}
Each layer receives its adjoint from \emph{exactly one term} in the
Neumann series: layer $\ell$ gets the $(L-\ell)$-th power of $\cB$.
This is the \emph{single-path collapse} of \Cref{rem:collapse}: in
a strictly feedforward network, there is exactly one path from the
output to each layer, so the summation in the Neumann series is
trivial.
\end{example}
\begin{example}[Vector-valued network with non-square weights]
\label{ex:vector1}
This example verifies that the framework handles rectangular weight
matrices correctly, which is essential for practical network
architectures with varying layer widths.
Let $L=3$ with $N_1=2$, $N_2=3$, $N_3=2$, identity activation
($\sigma'\equiv 1$, so $\Sigmap=I_7$), and weight matrices
\[
W^{(2)} = \begin{pmatrix} 1 & 0 \\ 0 & 1 \\ 1 & 1 \end{pmatrix}
\in \R^{3\times 2}, \quad
W^{(3)} = \begin{pmatrix} 1 & 1 & 0 \\ 0 & 1 & 1 \end{pmatrix}
\in \R^{2\times 3}.
\]
Take the terminal gradient $g = (1,1)^T \in \R^2$.
\paragraph{Layerwise computation.}
Since $\sigma'\equiv 1$, $Y_*^{(\ell)} = X_*^{(\ell)}$ for all $\ell$, so
\[
X_*^{(2)} = (W^{(3)})^T g
= \begin{pmatrix}1&0\\1&1\\0&1\end{pmatrix}\begin{pmatrix}1\\1\end{pmatrix}
= \begin{pmatrix}1\\2\\1\end{pmatrix},
\]
\[
X_*^{(1)} = (W^{(2)})^T X_*^{(2)}
= \begin{pmatrix}1&0&1\\0&1&1\end{pmatrix}\begin{pmatrix}1\\2\\1\end{pmatrix}
= \begin{pmatrix}2\\3\end{pmatrix}.
\]
\paragraph{Global verification.}
With $\cB_{1,2} = (W^{(2)})^T$ and $\cB_{2,3} = (W^{(3)})^T$, the
Neumann series gives
\[
\cB^0\bG = (0,0,0,0,0,1,1)^T,\quad
\cB^1\bG = (0,0,1,2,1,0,0)^T,\quad
\cB^2\bG = (2,3,0,0,0,0,0)^T,
\]
summing to
\[
\Xs = (2,3,1,2,1,1,1)^T,
\]
which partitions as $X_*^{(1)}=(2,3)^T$, $X_*^{(2)}=(1,2,1)^T$,
$X_*^{(3)}=(1,1)^T$ --- exactly matching the layerwise result. This
confirms \Cref{cor:equivalence} for non-square weight matrices.
\end{example}
\begin{example}[Residual connection breaks single-path collapse]
\label{ex:residual}
This example demonstrates the fundamental difference between
strictly feedforward networks and architectures with skip
connections.
\paragraph{Modified architecture.}
Add a skip connection from layer $1$ to layer $3$ to the scalar
network of \Cref{ex:scalar}, with scalar weight $S^{(31)}$. To keep
the layer-3 backward block $\cB_{2,3}=(W^{(3)})^Ts^{(3)}$ of
\Cref{ex:scalar} unchanged, the skip term is merged \emph{into the
pre-activation} $Y^{(3)}$ of the existing output nonlinearity,
rather than added directly to the post-activation $X^{(3)}$:
\begin{equation}\label{eq:residual-preact}
Y^{(3)} = w^{(3)}X^{(2)} + S^{(31)}X^{(1)}
= w^{(3)}\sigma\bigl(w^{(2)}X^{(1)}\bigr) + S^{(31)}X^{(1)},
\qquad X^{(3)} = \sigma\bigl(Y^{(3)}\bigr).
\end{equation}
(This is the pointwise-activation analogue of a \emph{pre-activation}
residual merge; it differs from the additive, activation-free block
$X^{(\ell)}=X^{(\ell-1)}+\mathcal G^{(\ell)}(X^{(\ell-1)})$ used in
\Cref{sec:applications}, and is why $s^{(3)}=\sigma'(Y^{(3)})$ appears
as a common factor on \emph{both} backward paths below; the
activation-free convention of \Cref{sec:applications} is the special
case $\sigma=\mathrm{Id}$ at the merge layer, i.e.\ $s^{(3)}\equiv1$.)
Applying the two-step rule~\eqref{eq:back1}--\eqref{eq:back2} to
\eqref{eq:residual-preact} gives $Y_*^{(3)}=X_*^{(3)}s^{(3)}=gs^{(3)}$,
$X_*^{(2)}=w^{(3)}Y_*^{(3)}$, and a \emph{direct} contribution
$S^{(31)}Y_*^{(3)}$ to $X_*^{(1)}$; the global backward operator
therefore gains an extra off-bidiagonal block:
\[
\cB = \begin{pmatrix}
0 & (W^{(2)})^T s^{(2)} & (S^{(31)})^T s^{(3)} \\
0 & 0 & (W^{(3)})^T s^{(3)} \\
0 & 0 & 0
\end{pmatrix}.
\]
This operator remains strictly upper-triangular (hence nilpotent)
but is \emph{no longer bidiagonal}.
\paragraph{Multiple path contributions.}
Layer $1$ now receives adjoint information via \emph{two distinct
routes}:
\[
X_*^{(1)} = \underbrace{(S^{(31)})^T s^{(3)} g}_{\text{skip path }(k=1)}
+ \underbrace{(W^{(2)})^T s^{(2)} (W^{(3)})^T s^{(3)} g}_{\text{sequential path }(k=2)}.
\]
The first term corresponds to the direct skip connection (one
application of $\cB$), while the second corresponds to the
sequential route through layer $2$ (two applications of $\cB$).
\paragraph{Interpretation.}
This algebraic decomposition formalises the known intuition that
residual networks behave like ensembles of paths~\cite{veit2016residual}.
The Neumann series summation, which was trivial for strict
feedforward networks, now becomes essential: each path through the
computational graph contributes one term to the sum. This is the
precise sense in which residual connections \emph{break the
single-path collapse} of \Cref{rem:collapse}.
\end{example}
\begin{remark}[Pedagogical summary of the examples]
\label{rem:examples-summary}
The four examples above progressively build intuition:
\begin{enumerate}[leftmargin=*,label=(\roman*)]
\item \Cref{ex:global-structure} demonstrates the full machinery
on a concrete network, verifying the three-way equivalence.
\item \Cref{ex:scalar} isolates the single-path collapse in the
simplest possible setting.
\item \Cref{ex:vector1} confirms the framework handles
rectangular weight matrices correctly.
\item \Cref{ex:residual} shows how architectural modifications
fundamentally alter the path structure.
\end{enumerate}
Together, these examples provide both a verification of the theory
and a pedagogical bridge from the abstract operator formulation to
practical network architectures.
\end{remark}
% ---------------------------------------------------------------
\section{Applications}\label{sec:applications}
% ---------------------------------------------------------------
The block-triangular, compositional structure of $\cB$ established in
\Cref{sec:methodology} has direct architectural consequences.
\subsection{Residual (skip-connection) networks}
A \emph{residual block}~\cite{he2016deep} computes
$X^{(\ell)} = X^{(\ell-1)} + \mathcal{G}^{(\ell)}(X^{(\ell-1)})$.
The F-adjoint propagates the cotangent as
$X_{*}^{(\ell-1)} = X_{*}^{(\ell)} + J_{\mathcal{G}^{(\ell)}}\transp X_{*}^{(\ell)}$.
The first term passes through the skip connection unmodified, giving
the \emph{gradient highway} effect. \Cref{ex:residual} exhibits the
corresponding global-operator picture: the skip connection introduces
an off-bidiagonal block in $\cB$, so the adjoint receives multiple
Neumann-series contributions, algebraically formalising the intuition
that residual networks behave like ensembles of paths~\cite{veit2016residual}.
\begin{corollary}[Path ensemble interpretation of residual networks]
\label{cor:residual-paths}
For a residual network with skip connections, the adjoint $X_{*}^{(\ell)}$
is a sum over all paths from layer $L$ to layer $\ell$ in the computational
graph, with each path contributing the product of the Jacobians along
its edges.
\end{corollary}
\begin{proof}
Skip connections add extra nonzero off-bidiagonal blocks
$\cB_{i,j}$ ($j>i+1$) to $\cB$ (as in \Cref{ex:residual}), but
$\cB$ remains block strictly upper-triangular: any block strictly
upper-triangular matrix on $L$ blocks is nilpotent with
$\cB^{L}=0$, by the same index-counting argument as in the proof
of \Cref{thm:nilpotency} (a nonzero $(i,j)$ block of $\cB^{k}$
still requires a strictly increasing chain of $k$ hops from $i$ to
$j$, hence $j-i\ge k$, hence $k\le L-1$), regardless of which
superdiagonal blocks are populated. \Cref{lem:path-general} then
applies with $\ell$ ranging over the layers reachable from $L$,
giving the stated path sum.
\end{proof}
\subsection{Transfer learning and fine-tuning}
Transfer learning partitions the network into a \emph{frozen encoder}
(Block 1) and a \emph{trainable head} (Block 2). By the two-block
decomposition, the global F-adjoint decomposes cleanly, and the backward
pass through the frozen encoder can be skipped entirely.
\begin{corollary}[Gradient truncation for frozen layers]
\label{cor:truncation}
Suppose the weights $W^{(\ell)}$, $\ell=1,\ldots,k$, are \emph{frozen}
(excluded from the optimiser), so that only $W_*^{(\ell)}$ for
$\ell=k+1,\ldots,L$ are actually needed. Computing every required
$W_*^{(\ell)}$, $\ell=k+1,\ldots,L$, requires evaluating the two-step
recursion~\eqref{eq:back1}--\eqref{eq:back2} only for layers
$\ell=L,\ldots,k+1$, i.e.\ $2(L-k)$ steps in total, versus $2L$
steps for a network trained end to end; the backward pass may stop
as soon as $X_*^{(k)}$ has been produced.
\end{corollary}
\begin{proof}
By \Cref{prop:path}, block $\ell$ of $\Xs$ is obtained after exactly
$L-\ell$ applications of $\cB$ to $\bG$, realised by
\Cref{prop:backsub} as $L-\ell$ back-substitution steps counted
from layer $L$. Forming $W_*^{(\ell)}=Y_*^{(\ell)}(X^{(\ell-1)})\transp$
for every $\ell=k+1,\ldots,L$ requires exactly the adjoint variables
$X_*^{(k)},\ldots,X_*^{(L)}$, obtained in $L-k$ back-substitution
steps of two sub-steps each -- $2(L-k)$ steps -- and no computation
involving $W^{(1)},\ldots,W^{(k)}$ is ever needed.
\end{proof}
% ---------------------------------------------------------------
\section{Discussion}\label{sec:discussion}
% ---------------------------------------------------------------
While the F-adjoint framework provides a rigorous global operator theory
for feedforward networks, it is essential to explicitly delineate its
assumptions and the boundaries of its current applicability.
\subsection{Structural consequences of the equivalence}
\begin{corollary}[Gradient norm control]
\label{cor:gradnorm}
The operator norm of the sensitivity satisfies
$\norm{(I-\cB)^{-1}} \le \sum_{k=0}^{L-1} \norm{\cB}^{k}$.
If $\norm{\cB}<1$ (contractive regime), this bound is uniform in $L$.
Nilpotency ensures the sum is always \emph{finite}, even when
$\norm{\cB}\ge 1$.
\end{corollary}
\begin{proof}
By \Cref{cor:invertible}, $(I-\cB)^{-1}=\sum_{k=0}^{L-1}\cB^{k}$.
The triangle inequality gives the bound. If $\norm{\cB}<1$, the
sum is bounded by $1/(1-\norm{\cB})$; if $\norm{\cB}\ge1$, the sum
has exactly $L$ terms, so it is finite.
\end{proof}
\subsection{Assumptions and limitations}
\begin{assumption}[Global differentiability]\label{asm:smooth}
The framework strictly assumes that the activation function $\sigma$
is everywhere differentiable. Modern deep learning heavily relies on
non-smooth activations such as the Rectified Linear Unit
(ReLU, $\sigma(x) = \max(0, x)$). While ReLU is almost everywhere
differentiable, the points of non-differentiability require
subdifferential calculus. Extending the nilpotent operator $\cB$ to
handle set-valued Jacobians remains an open challenge.
\end{assumption}
\begin{assumption}[Deterministic forward pass]\label{asm:deterministic}
The F-propagation assumes a deterministic mapping. Architectures
employing stochastic depth, Dropout, or sampling-based layers
introduce randomness into the forward pass. The global operator
$\cB$ would need to be treated as a random matrix, complicating
the exact termination of the Neumann series.
\end{assumption}
\subsection{Research gaps and future extensions}
\paragraph{Attention mechanisms and transformers.}
The current framework models strictly local, layer-to-layer dependencies.
Transformer architectures~\cite{vaswani2017attention} rely on global
self-attention, where every token attends to every other token. This
results in a dense Jacobian block, breaking the upper-bidiagonal
structure of $\cB$. While $\cB$ remains strictly upper-triangular
(if causal masking is applied) and thus nilpotent, the loss of sparsity
means the Neumann series no longer collapses to single-path terms.
\paragraph{Graph Neural Networks (GNNs).}
In GNNs~\cite{kipf2016semi}, the computational graph is arbitrary, not
a linear chain. The global operator $\cB$ would be indexed by the nodes
of the graph rather than layer indices. Nilpotency would be guaranteed
if and only if the computational graph is a Directed Acyclic Graph (DAG).
\paragraph{Connections to existing theory.}
The global fixed-point equation $(I-\cB)\Xs=\bG$ parallels the
adjoint-state equations of optimal control~\cite{pontryagin1987mathematical}.
Neural ODEs~\cite{chen2018neural} use continuous adjoint methods, while
the discrete formulation here achieves exactness through nilpotency;
deep equilibrium models~\cite{bai2019deep} solve implicit equations
requiring iterative methods, while finite-depth explicit networks admit
closed-form solutions via the finite Neumann series. In addition, \cite{scurria2026} builds on the global activation vector, weight matrix, and nilpotency characterisation developed here to derive a global Lagrangian dynamics that unifies the physical formalism of least action with backpropagation.
% ---------------------------------------------------------------
\section{Conclusion}\label{sec:conclusion}
% ---------------------------------------------------------------
We have given a self-contained, rigorously expanded treatment of the
F-propagation and F-adjoint framework. We proved that classical
layerwise backpropagation, the global linear system $(I-\cB)\Xs=\bG$,
and its solution by the Neumann series or by block back-substitution
are three equivalent descriptions of a single object. The global backward
propagation operator $\cB$ is strictly block upper-triangular and
nilpotent of index at most $L$, ensuring exact termination of the
Neumann series after at most $L$ terms. This unified perspective
clarifies why residual connections improve gradient flow (by providing
multiple path contributions) and why frozen encoders can be skipped
during fine-tuning (requiring only $2(L-k)$ backward steps).
By explicitly identifying the assumptions of differentiability and
determinism, we have outlined clear research gaps for extending this
elegant operator-theoretic framework to non-smooth activations,
stochastic architectures, and dense attention mechanisms. The framework
also connects naturally to adjoint methods in optimal control and to
deep equilibrium models, suggesting potential for cross-fertilisation
between these areas.
% --------------------------------------------------------------- % References % ---------------------------------------------------------------
\bibliographystyle{plain}
\bibliography{references}

@inproceedings{bai2019deep,
	author    = {Bai, Shaojie and Kolter, J. Zico and Koltun, Vladlen},
	title     = {Deep Equilibrium Models},
	booktitle = {Advances in Neural Information Processing Systems},
	volume    = {32},
	year      = {2019}
}

@article{basheer2000artificial,
	author  = {Basheer, I. A. and Hajmeer, M.},
	title   = {Artificial Neural Networks: Fundamentals, Computing, Design, and Application},
	journal = {Journal of Microbiological Methods},
	volume  = {43},
	number  = {1},
	pages   = {3--31},
	year    = {2000}
}

@article{baydin2018automatic,
	author  = {Baydin, A. G. and Pearlmutter, B. A. and Radul, A. A. and Siskind, J. M.},
	title   = {Automatic Differentiation in Machine Learning: A Survey},
	journal = {Journal of Machine Learning Research},
	volume  = {18},
	number  = {153},
	pages   = {1--43},
	year    = {2018}
}

@article{boughammoura2023two,
	author  = {Boughammoura, A.},
	title   = {A Two-Step Rule for Backpropagation},
	journal = {International Journal of Informatics and Applied Mathematics},
	volume  = {6},
	number  = {1},
	pages   = {57--69},
	year    = {2023}
}

@article{boughammoura2023,
	author  = {Boughammoura, A.},
	title   = {Backpropagation and F-adjoint},
	journal = {arXiv preprint},
	eprint  = {2304.13820},
	archivePrefix = {arXiv},
	year    = {2023}
}

@article{boughammoura2024learning,
	author  = {Boughammoura, A.},
	title   = {Learning by the F-adjoint},
	journal = {arXiv preprint},
	eprint  = {2407.11049},
	archivePrefix = {arXiv},
	year    = {2024}
}

@inproceedings{chen2018neural,
	author    = {Chen, Ricky T. Q. and Rubanova, Yulia and Bettencourt, Jesse and Duvenaud, David K.},
	title     = {Neural Ordinary Differential Equations},
	booktitle = {Advances in Neural Information Processing Systems},
	volume    = {31},
	year      = {2018}
}

@book{goodfellow2016deep,
	author    = {Goodfellow, Ian and Bengio, Yoshua and Courville, Aaron},
	title     = {Deep Learning},
	publisher = {MIT Press},
	address   = {Cambridge, MA},
	year      = {2016}
}

@book{griewank2008evaluating,
	author    = {Griewank, Andreas and Walther, Andrea},
	title     = {Evaluating Derivatives: Principles and Techniques of Algorithmic Differentiation},
	publisher = {SIAM},
	address   = {Philadelphia},
	year      = {2008}
}

@inproceedings{he2016deep,
	author    = {He, Kaiming and Zhang, Xiangyu and Ren, Shaoqing and Sun, Jian},
	title     = {Deep Residual Learning for Image Recognition},
	booktitle = {Proceedings of the IEEE Conference on Computer Vision and Pattern Recognition},
	pages     = {770--778},
	year      = {2016}
}

@book{hinze2008optimization,
	author    = {Hinze, Michael and Pinnau, Rene and Ulbrich, Michael and Ulbrich, Stefan},
	title     = {Optimization with PDE Constraints},
	publisher = {Springer},
	address   = {Dordrecht},
	year      = {2008}
}

@article{kipf2016semi,
	author  = {Kipf, Thomas N. and Welling, Max},
	title   = {Semi-supervised Classification with Graph Convolutional Networks},
	journal = {arXiv preprint},
	eprint  = {1609.02907},
	archivePrefix = {arXiv},
	year    = {2016}
}

@article{lecun1989backpropagation,
	author  = {LeCun, Yann and Boser, Bernhard and Denker, John S. and Henderson, Donnie and Howard, Richard E. and Hubbard, Wayne and Jackel, Lawrence D.},
	title   = {Backpropagation Applied to Handwritten Zip Code Recognition},
	journal = {Neural Computation},
	volume  = {1},
	number  = {4},
	pages   = {541--551},
	year    = {1989}
}

@article{mcculloch1943,
	author  = {McCulloch, Warren S. and Pitts, Walter},
	title   = {A Logical Calculus of the Ideas Immanent in Nervous Activity},
	journal = {Bulletin of Mathematical Biophysics},
	volume  = {5},
	pages   = {115--133},
	year    = {1943}
}

@book{pontryagin1987mathematical,
	author    = {Pontryagin, L. S. and Boltyanskii, V. G. and Gamkrelidze, R. V. and Mishchenko, E. F.},
	title     = {The Mathematical Theory of Optimal Processes},
	publisher = {Gordon and Breach},
	address   = {New York},
	year      = {1987}
}

@article{rumelhart1986,
	author  = {Rumelhart, David E. and Hinton, Geoffrey E. and Williams, Ronald J.},
	title   = {Learning Representations by Back-propagating Errors},
	journal = {Nature},
	volume  = {323},
	number  = {6088},
	pages   = {533--536},
	year    = {1986}
}

@article{scurria2026,
	author  = {Scurria, Anthony E.},
	title   = {A Physical Theory of Backpropagation: Exact Gradients from the Least-Action Principle},
	journal = {arXiv preprint},
	eprint  = {2602.02281},
	archivePrefix = {arXiv},
	year    = {2026}
}

@inproceedings{vaswani2017attention,
	author    = {Vaswani, Ashish and Shazeer, Noam and Parmar, Niki and Uszkoreit, Jakob and Jones, Llion and Gomez, A. N. and Kaiser, Lukasz and Polosukhin, Illia},
	title     = {Attention Is All You Need},
	booktitle = {Advances in Neural Information Processing Systems},
	volume    = {30},
	year      = {2017}
}

@inproceedings{veit2016residual,
	author    = {Veit, Andreas and Wilber, Michael J. and Belongie, Serge},
	title     = {Residual Networks Behave Like Ensembles of Relatively Shallow Networks},
	booktitle = {Advances in Neural Information Processing Systems},
	volume    = {29},
	year      = {2016}
}

@book{ye2022,
	author    = {Ye, Jong Chul},
	title     = {Geometry of Deep Learning},
	publisher = {Springer},
	address   = {Singapore},
	year      = {2022}
}
\end{document}